\documentclass[10pt,conference]{IEEEtran}
\usepackage[
	disable,
backgroundcolor = White,textwidth=\marginparwidth]{todonotes}

\usepackage[markup=underlined]{changes}

\usepackage[utf8]{inputenc} 
\usepackage[T1]{fontenc}    
\usepackage{url}            
\usepackage{booktabs}       
\usepackage{amsfonts}       
\usepackage{nicefrac}       
\usepackage{microtype}      

\usepackage{graphicx}
\usepackage{color}
\usepackage{rotating}
\usepackage{tabularx}
\usepackage{pdflscape}
\usepackage{amsmath,amsthm,amssymb}
\usepackage{blkarray}
\usepackage{algorithm,algorithmic}
\usepackage{bbm,dsfont}
\usepackage{caption,subcaption}
\usepackage{wrapfig}
\usepackage{cancel}
\usepackage{enumerate, cases}
\usepackage{thmtools,thm-restate}
\usepackage{mathtools}

\usepackage[none]{hyphenat}
\usepackage{multirow}
\usepackage{makecell}
\usepackage{setspace}

\allowdisplaybreaks

\usepackage{dblfloatfix}
\usepackage{array}
\usepackage{enumitem}

\allowdisplaybreaks

\DeclareMathOperator*{\argmax}{arg\,max}

\newcolumntype{C}[1]{>{\centering}m{#1}}
\newcommand{\EE}[1]{\mathbb{E}\left[#1\right]}

\newcommand{\Prob}[1]{\mathbb{P}\left\{#1\right\}}

\newcommand{\Regret}{\mathcal{R}}
\newcommand{\A}{\mathcal{A}_{C}}

\theoremstyle{plain}
\newtheorem{thm}{Theorem}
\newtheorem{lem}{Lemma}

\IEEEoverridecommandlockouts
\IEEEpubid{\makebox[\columnwidth]{978-1-7281-8895-9/20/\$31.00~\copyright~2020~IEEE \hfill}
\hspace{\columnsep}\makebox[\columnwidth]{ }}

\begin{document}

    \title{Learning and Fairness in Energy Harvesting:	\\A Maximin Multi-Armed Bandits Approach}
    \author{Debamita Ghosh, Arun Verma,  and Manjesh K. Hanawal \\
	Industrial Engineering and Operations Research\\
	Indian Institute of Technology Bombay, Mumbai, India - 400076\\
	\{debamita.ghosh, v.arun, mhanawal\}@iitb.ac.in
	}
    \maketitle

    \begin{abstract}
        Recent advances in wireless radio frequency (RF) energy harvesting allows sensor nodes to increase their lifespan by remotely charging their batteries. The amount of energy harvested by the nodes varies depending on their ambient environment, and proximity to the energy source, and lifespan of the sensor network depends on the minimum amount of energy a node can harvest in the network. It is thus important to learn the least amount of energy harvested by nodes so that the source can transmit on a frequency band that maximizes this amount. We model this learning problem as a novel stochastic \textit{Maximin Multi-Armed Bandits} (Maximin MAB) problem and propose an Upper Confidence Bound (UCB) based algorithm named \textit{Maximin UCB}. Maximin MAB is a generalization of standard MAB, and Maximin UCB enjoys the same performance guarantee as to the UCB1 algorithm. Our experimental results validate the performance guarantees of the proposed algorithm.
    \end{abstract}

    \begin{IEEEkeywords}
    		Multi-Armed Bandits, Upper Confidence Bound, Radio Frequency, Energy Harvesting, Fairness
    \end{IEEEkeywords}

    \bstctlcite{IEEEexample:BSTcontrol}

    \section{Introduction}
    \label{sec:introduction}

The recent advances of radio frequency energy harvesting (RFEH) networks have emerged as a feasible option for the next-generation wireless communication networks. Wireless systems can increase their lifespan and be environmentally friendly by extracting RF energy from natural ecosystems or by dedicated energy sources. In this work, we focus on RFHN, where a dedicated energy source transmits energy that wireless sensors nodes (WSN) can harvest to charge their batteries. This setup arises in many IoT systems where battery-powered sensors are deployed in remote environments, and an energy source can keep them active through wireless charging from a point where uninterrupted power supply is available.

Energy harvested at WSN depends on the circuits used, their performance, processing capabilities, and these may vary depending on the ambient conditions in which the sensors operate. Also, the amount of energy harvested depends on the frequency bands. The energy source can send energy on the entire band or divide it into sub-bands and concentrate all power on one of the sub-bands, which can improve the RF potential of the bands \cite{AWPL2014_li2014frequency,CM2015_mishra_smart}. The energy source has to decide then on which frequency sub-band to transmit energy so that the amount of energy harvested by the WSN is maximized. However, the amount of energy harvested on each of the sub-bands could be unknown and has to be first learned by the source. Henceforth, we refer to the frequency sub-bands as channels.
 
In sensors networks, all the nodes must be kept alive so that all of them can transmit information. However, the nodes may be at different locations, and the amount of energy harvested may be different. Then the energy source has to be `fair' in selecting a channel for energy transmission so that all the nodes can harvest energy. While one can look at many fairness criteria, one possibility is that the source transmits on the maximin optimal channel where the smallest energy harvested by a node is maximized, i.e., the source selects a channel based on maximin framework. Following this criterion, the source can ensure that all the nodes are active as long as possible.


Due to the stochastic nature of the wireless channels, the amount of energy harvested by nodes on each channel could be random. Further, the energy source may not know a priori the distribution of the amount of energy harvested by the nodes on a channel, and hence the source is faced with decision making in an uncertain environment. We set up our problem as a Multi-Armed Bandits (MAB) problem, where we refer the energy source as a learner and the channels as arms. We model our problem as Maximin MAB, where the goal is to identify a frequency band (channel) that is maximin optimal. In contrast to the classical MAB set up \cite{Book19_Tor19Bandits}, which identifies an arm (channel) with the highest average reward, in Maximin MAB set up, the goal is to identify the channel that maximizes the minimum average rewards by any node. As we will see later, our structure generalizes the standard MAB setup for vector-valued rewards and, as a special case, includes the standard MAB when there is only one node in the network. Specifically, our contributions can be summarized as follows:
\begin{itemize}
 	\item In Section \ref{sec:problem_setting}, we introduce the Maximin MAB set up to study the aspects of learning and fairness in RFHN with a dedicated energy source.
 	\item We propose an Upper Confidence Bound (UCB) based algorithm named as \ref{alg:MaxMin-UCB} for the new setup and provide its performance guarantee in Section \ref{sec:algorithm}.
 	\item We empirically validate the performance of \ref{alg:MaxMin-UCB} on the synthetic problem instances in Section \ref{sec:experiment}.  
\end{itemize}

\noindent
\textbf{Related Work:}
Various aspects of RFEH networks are studied in the literature. Here we discuss the learning and fairness related issues, which is the focus of this work.
For a detailed survey on RFEH, we refer to \cite{JSAC2015_energyharvesting_UlukusSennurYener, CST2014_wirelessnetwork_LuXiaoPingDusit, MM2014_harvesting_ValentaChristGreg}.
The authors of \cite{TIT2017_online_sakulkar} study the problem of rate maximization in EH communication with unknown channel states and develop a learning policy to maximize the rate achieved by the transmitter, modeling it as Markov Decision Process (MDP). A policy using the Bayesian MAB algorithm to select frequency bands based on their RF potential, particularly in the dynamic spectrum setting, is studied in \cite{WN2018_darak2018distributed}. The author in \cite{IA2017_maghsudi_distributed} proposes a multi-armed mean-field bandit based distributed approach for user association in an EH dense small scale network. In \cite{AA-MAS2012_tran_long}, authors consider the problem of energy management and data routing to maximize the information collected under the energy budget. For energy management, they use MAB based learning for allocated energy to sample, receive, and transmit. For network with multiple players, \cite{INFOCOM2019_DistributedLearning_TibrewalPatchalaHanawal}, \cite{WiOpt2019_DistributedAlgorithms_VermaHanawalVaze} use multi player MAB setting for distibuted learning.

The fairness issues in RFEH are considered in \cite{TWC2013_ju2013throughput, GCC2014_ju2014user,JSAC2015_yang2015throughput}. In networks where EH nodes receive energy from a source and transmit back information, unfair rate allocation occurs as nodes far from the source receive less energy, but has to use more energy for transmission (doubly-near-far problem). The authors in \cite{TWC2013_ju2013throughput, GCC2014_ju2014user} propose fairness constraints named common-throughput so that all the nodes achieve the same throughput. In \cite{JSAC2015_yang2015throughput}, the authors consider fair rate allocation in a massive MIMO RFEH network where a Hybrid Access Point (H-AP) wirelessly charges the nodes on the downlink and receives data from them on the uplink. The authors formulated a scheme that asymptotically maximizes the minimum rate among all the nodes. Our work differs from the existing literature as we consider both learning and fairness issues together.

    \section{Problem Setting}
    \label{sec:problem_setting}

We consider an energy source (referred simply as source) that wirelessly charges $p$ nodes. The source divides its available bandwidth for energy transmission into $m$ channels. We assume that the source transmits a fixed amount of power at any time on one of the channels. The amount of energy harvested by nodes on a channel is stochastic and depends on their location, distance from the source, and hardware capabilities. The energy harvesting setup with an energy source wirelessly charging $5$ nodes using $3$ channels is depicted in Fig. \ref{fig:Model}.
\begin{figure}[H]
	\centering
	\includegraphics[width=0.5\linewidth]{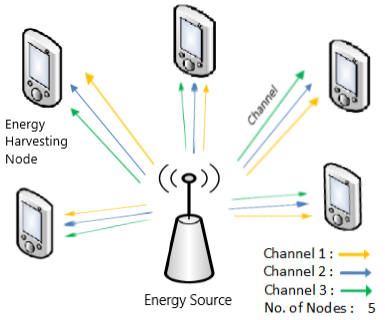}
	\caption{\small Energy Harvesting setup. Refer \cite{CST2014_wirelessnetwork_LuXiaoPingDusit} for a detailed architecture.}
	\label{fig:Model}
\end{figure}
We assume that the time is slotted, and in each slot, the source decides which channel to use for wireless charging. At the end of the time slot, the nodes inform the source of how much energy they could harvest in that round\footnote{Nodes can provide feedback by sending back the current state of their battery to the source.}. The source uses the feedback from the nodes to decide which channel to use in the next round. The goal of the source is to select a maximin optimal channel in which the minimum average energy harvested by any node is maximized.

We model the problem as Maximin MAB as follows: Let $m$ be the number of channels, $p$ be the number of nodes and $X^{(t)}_{ij}$ denote the energy harvested by node $j\in [p]$ on channel $i \in [m]$ in round $t$. For each channel-node pair $(i,j)$ where $i\in [m],j\in [p]$, the sequence $\{X^{(t)}_{ij}\}_{t\geq 1}$ is drawn independently and identically from an unknown distribution with mean $\mu_{ij}$. Further, these sequences are independent across $i$ and $j$. We assume that the distributions associated with each channel-node pair are sub-Gaussian with parameter $\sigma$ where $\sigma>0$. Thus a problem instance of the Maximin MAB is identified by a mean matrix given as follows:
\[
\begin{blockarray}{ccccc}
& \text{Node } 1 & \text{Node } 2 & \ldots & \text{Node }p \\
\begin{block}{c(cccc)}
  \text{Channel }1 & \mu_{11} & \mu_{12} & \ldots & \mu_{1p}\\
  \text{Channel } 2 & \mu_{21} & \mu_{22} & \ldots & \mu_{2p} \\
  \cdot & \cdot & \cdot & \ldots & \cdot \\
  \cdot & \cdot & \cdot & \ldots & \cdot \\
  \text{Channel } m & \mu_{m1} & \mu_{m2} & \ldots & \mu_{mp} \\
\end{block}
\end{blockarray}
 \]
\vspace{-4mm}

In our setup, the interaction between the source and the environment that governs the rewards for the channel-node pairs are as follows: In round $t$, the source selects the channel $I_t$, and receives the energy harvested by all the nodes on the channel $I_t$, i.e., the reward vector $\{X^{(t)}_{I_{t}1}, X^{(t)}_{I_{t}2},\ldots, X^{(t)}_{I_{t}p}\}$ as feedback. The goal of the source is to select a channel that maximizes the minimum average energy harvested by any node which is given as follows: 
\begin{equation}
    	i^{*} = \argmax\limits_{i \in [m]}(\min\limits_{j \in [p]} \mu_{ij}). \nonumber 
\end{equation}

A policy of the source consists of selecting a channel in each round based on past observations. Let $\pi$ denote a policy that selects the channel $I_t$ in round $t$. Then, we define the regret of Maximin MAB problem for $n$ rounds as follows:
\begin{equation}
	\label{equ:regret}
	\Regret_{n} = n\mu_{i^{*}} - \sum\limits_{t=1}^{n}\EE{\min\limits_{j \in [p]}\mu_{I_{t}j}}
\end{equation}
where $\mu_{i^{*}}= \min_{j \in [p]} \mu_{i^{*}j}$ and the expectation is with respect to the randomness in $I_t$.
We say that the policy is good if the regret is sub-linear, i.e., $\Regret_{n}/n \rightarrow 0$ as $n \rightarrow \infty$.  


    \section{Algorithm}
    \label{sec:algorithm}

We develop an algorithm named \ref{alg:MaxMin-UCB} that adapts the UCB1 algorithm \cite{ML2002_FiniteTimeAnlaysis_Auer} to our setting. The pseudo-code of the proposed algorithm is given in \ref{alg:MaxMin-UCB}. 

Recall that in our setup, selecting a channel gives feedback from all the nodes. Let $i \in [m]$, $j \in [p]$, and $T_i(t)$ be the number of times the channel $i$ is selected till time $t$. At the beginning of round $t$, the empirical mean energy harvested by node $j$ on channel $i$ is computed using $T_i(t-1)$ samples, and it is denoted by $\hat{\mu}_{ij}(t-1)$.


\ref{alg:MaxMin-UCB} works as follows: It takes $m, p, \sigma$, and $\delta$ as inputs, where $\sigma$ is the sub-Gaussian parameter, and $\delta$ is the confidence parameter. In the first $m$ rounds, each channel is selected in a round-robin fashion. In the subsequent round $t$, the UCB index is calculated for each channel $i$ denoted as UCB$_i(t, \delta)$. The channel with the highest value of UCB$_i(t, \delta)$ is selected, and corresponding estimates of $\hat{\mu}_{ij}$ are updated. 




\begin{algorithm}[H] 
	\renewcommand{\thealgorithm}{Maximin UCB}
	\floatname{algorithm}{}
	\caption{\bf }
	\label{alg:MaxMin-UCB}
	\begin{algorithmic}[1]
			\STATE \textbf{Input:} $m, p, \sigma, \delta$
			\STATE Select each channel once in first $m$ rounds
			\STATE Update $\hat{\mu}_{ij}$ for all $i\in [m], j\in [p]$
		\FOR{$t= m+1, m+2,\ldots,n$}
			\STATE For each channel $i\in [m]$ calculate 
				\begin{equation*}
 				    \text{UCB}_{i}(t, \delta)  \leftarrow \min\limits_{j \in [p]} \hat{\mu}_{ij}(t-1)+\sqrt{\frac{2\sigma^{2}\log(1/\delta)}{T_{i}(t-1)}} 	
				\end{equation*}
			\STATE Set $I_{t} \leftarrow \argmax\limits_{i \in [m]}\text{UCB}_{i}(t,\delta)$
			\STATE $\forall j \in [p]:$ Observe $X_{I_{t}j}$ and update the estimate of ${\mu}_{I_{t}j}$
		\ENDFOR
	\end{algorithmic}
\end{algorithm}




\subsection{Regret Analysis}
The regret for any policy $\pi$ can be decomposed as  $\Regret_{n} = \sum_{i \in [m]}\Delta_{i}\EE{T_{i}(n)}$, where $\Delta_{i} \doteq \mu_{i^{*}}- \min_{j \in [p]}\mu_{i j}$, called as the sub-optimality gap. Notice that when $p=1$, our Maximin MAB setup reduces to the standard MAB setup. Now we are ready to give theoretical guarantee for \ref{alg:MaxMin-UCB}. 
\begin{thm}
    \label{thm: Theorem 1}
    Let \ref{alg:MaxMin-UCB} runs for $n$ rounds and $\delta = {1}/{n^{2}}$. Then the regret of \ref{alg:MaxMin-UCB} on an instance $\mu:=\{\mu_{ij}\}_{i\in [m], j\in [p]}$ is upper bounded by
	\begin{equation*}
		\Regret_{n} \leq 3\sum_{i=1}^{m}\Delta_{i} + \sum_{i: \Delta_{i} > 0} \frac{16\sigma^{2}\log(n)}{\Delta_{i}}.
	\end{equation*}
\end{thm}

\noindent
Next we give a problem independent upper bound on regret.
\begin{thm}
    \label{thm: Theorem 2}
    Let \ref{alg:MaxMin-UCB} runs for $n$ rounds and $\delta = {1}/{n^{2}}$. Then its regret for any instance is upper bounded by 
	\begin{equation*}
	    \Regret_{n} \leq 8\sqrt{nm\sigma^{2}\log(n)} + 3\sum\limits_{i=1}^{m}\Delta_{i}.
	\end{equation*}
\end{thm}
Note that both of the above bounds do not depend on the number of nodes $(p)$ as we observe samples from all the nodes for the selected arm. 

    \section{Experiments}
    \label{sec:experiment}

We empirically evaluate the performance of \ref{alg:MaxMin-UCB} for two sets of experiments. In the first experiment, we compare the behavior of regret on varying minimum sub-optimality gap, defined as  $\Delta_{min} = \min_{i \in [m] \setminus i^{*}} \Delta_{i}$. Whereas the regret behaviors on different numbers of channels and nodes are compared in the second experiment. We initially set our energy harvesting setup for channels $m=6$ and nodes $p=5$ where each $(i,j)$ channel-node pair has Bernoulli distribution with the mean energy harvested as $\mu_{ij}$ where $\mu_{ij}=0.5 - 0.05(i-j)$.
We measure the performance of our policy over a time horizon $T = 50000, \delta = 1/T$ and $\sigma = 1$. We repeat each experiment 1000 times and present the average regret with a 95\% confidence interval, shown as the vertical line on each curve.

\vspace{2mm}
\noindent
\textbf{Regret v/s Minimum sub-optimality gap:}\hspace{1mm} We investigate the impact of regret on different minimum sub-optimality gaps ($0.03, 0.04, 0.05, 0.06, 0.07$). As expected, the regret decreases as the minimum sub-optimality gap increases (see Fig. \ref{fig:deltaRegret}). 
\vspace{-2.5mm}
\begin{figure}[H]
	\centering
	\begin{subfigure}[b]{0.239\textwidth}
		\includegraphics[width=\linewidth]{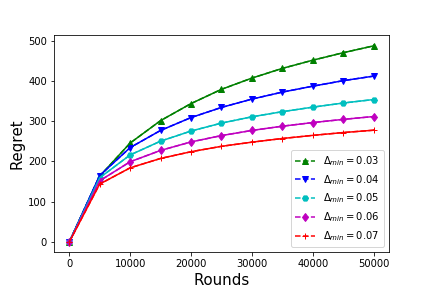}
    	\caption{Regret v/s $\Delta_{min}$}
    	\label{fig:deltaRegret}
	\end{subfigure}
	\begin{subfigure}[b]{0.239\textwidth}
		\includegraphics[width=\linewidth]{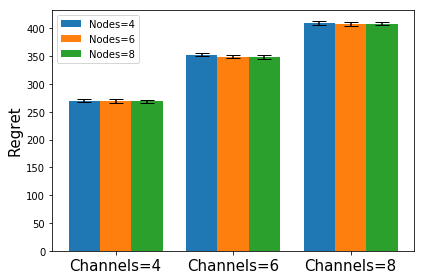}
		\caption{Regret v/s Channels and Nodes}
		\label{fig:RealLift}
	\end{subfigure}
	\caption{Empirical performance of \ref{alg:MaxMin-UCB}}
	\label{fig:regret 1}
\end{figure}
\vspace{-2mm}

\noindent
\textbf{Regret v/s Number of Channels and Nodes:}\hspace{2mm}  We run this experiment on the different number of channels $(4, 6, 8)$ and different number of nodes $(4, 6, 8)$. 
According to Fig. \ref{fig:RealLift}, the average regret is almost equal for different numbers of nodes for a fixed number of channels. But the regret increases as we increase the number of channels for a fixed number of nodes. As per our theoretical analysis, the simulated results also verify that regret is independent of the number of nodes.

    \section{Conclusion and Future Extensions}
    \label{sec:conclusion}
We considered radio frequency energy harvesting (RFEH) sensor network with a dedicated energy source. The source can transmit energy to the sensor node on one of the frequency bands, and the amount of energy harvested by the nodes is random with an unknown distribution. The goal is to ensure every node gets energy and remains active for a longer duration. We thus addressed learning and fairness issues in RFEH sensor networks. We cast the problem as Maximin Multi-Armed Bandits, where the goal is to identify a frequency band (arm) on which minimum mean energy harvested by each node is maximized. We developed an algorithm named Maximin UCB for this setup and showed that it enjoys the same regret guarantee as that of the well known UCB algorithm. 

In our study, we have ignored the current energy requirement of each node. In the future, it will be interesting to study the problem where the source transmits on a band based on the current battery levels of the nodes and the amount of information they have to transmit. As the amount of energy harvested by nodes is often non-linear function of power level, in addition to the selection of bands, the source may also decide on power levels to transmit energy so that the efficiency of harvested energy improves.

    \section*{Acknowledgment}
    \label{acknowledgement}
    Debamita Ghosh would like to thank Cognizant for funding under the "Umbrella Sponsorship Agreement" with IITB-Monash Research Academy, IIT Bombay. Manjesh K. Hanawal would like to thank the support from INSPIRE faculty fellowship from DST and Early Career Research (ECR) Award from SERB, Govt. of India.

    \section{Appendix}
    \label{sec:appendix}
\begin{lem}{\cite[Corollary 5.5]{Book19_Tor19Bandits}}
\label{lem:chernoffBound}
Let $X_{i} - \mu$ be the independent $\sigma-$sub-Gaussian random variables and $\hat{\mu} \doteq \frac{1}{n}\sum\limits_{t=1}^{n}X_{t}$. Then for any $\epsilon \geq 0$,
\begin{equation*}
    \Prob{|\hat{\mu} - \mu| \geq \epsilon} \leq 2\exp\left(-\frac{n\epsilon^{2}}{2\sigma^{2}}\right).
\end{equation*}
\end{lem}


\begin{proof}[\bf Proof of Theorem \ref{thm: Theorem 1}]
    Let $\{X_{I_{t}j}^{(t)}: t \in [n], i \in [m], j \in [p]\}$ be the collection of independent random variables.
    We define $\hat{\mu}_{ij}^{(s)} = \frac{1}{s}\sum_{r=1}^{s}X_{ij}^{(r)}$ as the empirical mean of $(i,j)^{th}$ channel-node pair 
    based on first $s$ samples, and $\hat{\mu}_{ij}(t)$ as the empirical mean of the $(i,j)^{th}$ channel-node pair 
    after round $t$. After selecting channel $I_{t}$ at round $t$, we obtain the reward $X_{I_{t}j}^{(t)}$ for the $(I_t,j)^{th}$ channel-node pair.
    Following the standard arguments, expected regret can be decomposed as $\Regret_{n} = \sum_{i=1}^{m}\Delta_{i}\EE{T_{i}(n)}$ where $\Delta_{i} = \mu_{i^{*}} - \min_{j \in [p]}\mu_{ij}$. After selecting each channel once by the algorithm, channel $i$ can only be selected if its index is higher than an optimal channel, which happens if at least one of the following holds:
    \begin{enumerate}[label=(\roman*)]
        \item  The channel index $i$ is greater than the true mean of the particular optimal channel.
        \item The optimal channel index is smaller than its true mean.
    \end{enumerate}
    The index of the ideal channel is always assumed to be above its mean, as it is an upper bound on its mean with reasonably high probability. However, when the sub-optimal channel $i$  is played often enough, its exploration bonus is low. At the same time, the empirical estimate of its mean converges to the true value, with an upper limit on the expected total number of times when its index stays above the optimal channel's mean. Without loss of generality, let the first channel is optimal i.e. $\mu_{1} = \max_{i \in [m]}\min_{j \in [p]}\mu_{ij} = \min_{j \in [p]}\mu_{1j}$. We will bound $\EE{T_{i}(n)}$ for each sub-optimal channel $i$. 
    
    Let us define a good event $G_{i}$ such that the first channel's upper confidence limit never underestimates the true mean of the optimal channel i.e. $\mu_{1}$, and the upper confidence bound of the channel $i$ based on $u_{i}$ number of observations is below the optimum channel's payoff.
    \begin{align}
        \label{equ:goodEvent}
        G_{i} =& \left\{\mu_{1} < \min_{t \in [n]}\text{UCB}_{1}(t)\right\} \bigcap \nonumber\\ 
        &\left\{\min_{j \in [p]}\hat{\mu}_{ij}^{(u_{i})} + \sqrt{\frac{2\sigma^{2}\log(1/\delta)}{u_{i}}} < \mu_{1}\right\}
    \end{align}
    where $u_{i} \in [n]$ is a constant to be chosen later.
    
    To show: (i) If $G_{i}$ occurs then $T_{i}(n) \leq u_{i}$.
    
    \qquad \qquad (ii) $G_{i}^{c}$ occurs with low probability.
    \begin{align*}
        \text{Define: }I(G_i) & =\begin{cases}
        1      & \textrm{if $G_i$ occurs for channel $i$}  \\
        0      & \textrm{otherwise}.
         \end{cases}
    \end{align*}
    \begin{align}
    \label{equ:expT}
    	\EE{T_{i}(n)} &= \EE{I(G_{i})T_{i}(n)} + \EE{I(G_{i}^{c})T_{i}(n)} \nonumber \\
    &  \leq u_{i} + n\Prob{G_{i}^{c}} \qquad \text{since $T_{i}(n) \leq n$.}
    \end{align}
    
    \underline{Claim 1:} $G_{i}$ occurs $\Rightarrow T_{i}(n) \leq u_{i}$.
    
    Let $T_{i}(n) > u_{i}$ i.e. the channel $i$ played more than $u_{i}$ times over the $n$ rounds and $\exists$ $t \in [n]$ $\ni$ $T_{i}(t-1) = u_{i}$ and $I_{t} = i$.
    Following the definition of $\text{UCB}_{i}(t-1)$ and $G_{i}$,
    \begin{align*}
    	&\text{UCB}_{i}(t-1) = \min_{j \in [p]}\hat{\mu}_{ij}(t-1) + \sqrt{\frac{2\sigma^{2}\log(1/\delta)}{T_{i}(t-1)}} \nonumber \\
        &\qquad= \min_{j \in [p]}\hat{\mu}_{ij}^{(u_{i})} + \sqrt{\frac{2\sigma^{2}\log(1/\delta)}{u_{i}}}  \hspace{0.3cm}\text{[$T_{i}(t-1) = u_{i}$]}\\
        &\qquad < \mu_{1}  \hspace{0.3cm}< \text{UCB}_{1}(t-1).
    \end{align*}
    This contradicts that $I_t=i$ in round $i$. Hence $I_i(n)\leq u_i$.
    \underline{Claim 2:} We next upper bound $\Prob{G_{i}^{c}}$. From (\ref{equ:goodEvent}), 
    \begin{align*}
        G_{i}^{c}= &\left\{\mu_{1} \geq \min\limits_{t \in [n]}\text{UCB}_{1}(t)\right\} \bigcup \nonumber\\ &\left\{\min\limits_{j \in [p]}\hat{\mu}_{ij}^{(u_{i})} + \sqrt{\frac{2\sigma^{2}\log(1/\delta)}{u_{i}}} \geq \mu_{1}\right\}.
    \end{align*}
    
    \begin{itemize}
        \item We provide an upper bound of the first component.
        \begin{align}
    	  &\left\{\mu_{1} \geq \min\limits_{t \in [n]}\text{UCB}_{1}(t)\right\} \nonumber \\
    	   &\subset \left\{\mu_{1} \geq \min\limits_{s\in[n]} \left(\min\limits_{j \in [p]}\hat{\mu}_{1j}^{(s)} + \sqrt{\frac{2\sigma^{2}\log(1/\delta)}{s}}\right)\right\} \nonumber \\
           &= \bigcup\limits_{s\in [n]}\left\{\mu_{1} \geq \min\limits_{j \in [p]}\hat{\mu}_{1j}^{(s)} + \sqrt{\frac{2\sigma^{2}\log(1/\delta)}{s}}\right\}. \nonumber\\
        \intertext{Applying probability on both sides we get,} \nonumber
            &\Prob{\mu_{1} \geq \min\limits_{t \in [n]}\text{UCB}_{1}(t)} \nonumber \\
        &\leq \Prob{\bigcup\limits_{s\in [n]}\left(\mu_{1} \geq \min\limits_{j \in [p]}\hat{\mu}_{1j}^{(s)} + \sqrt{\frac{2\sigma^{2}\log(1/\delta)}{s}}\right)} \nonumber \\
        &\leq \sum\limits_{s=1}^{n}\Prob{\mu_{1} \geq \min\limits_{j \in [p]}\hat{\mu}_{1j}^{(s)} + \sqrt{\frac{2\sigma^{2}\log(1/\delta)}{s}}}. \nonumber \\
        \intertext{Since nodes are independent for each channel,}
        &\leq \sum\limits_{s=1}^{n}\prod\limits_{j=1}^{p}\Prob{\mu_{1j} \geq \hat{\mu}_{1j}^{(s)} + \sqrt{\frac{2\sigma^{2}\log(1/\delta)}{s}}} \nonumber \\
        &\leq \sum\limits_{s=1}^{n}\Prob{\mu_{1j_{s}^{*}} \geq \hat{\mu}_{1j_{s}^{*}}^{(s)} + \sqrt{\frac{2\sigma^{2}\log(1/\delta)}{s}}}. \nonumber
    \end{align}
    Applying Lemma \ref{lem:chernoffBound} to this we get,
    \begin{align}
    \label{equ:firstcomponent}
    	\Prob{\mu_{1} \geq \min\limits_{t \in [n]}\text{UCB}_{1}(t)} \leq n\delta.
    \end{align}
    
     \item We provide an upper bound of the second component.
     
     We choose $u_{i}$ sufficiently large such that,
    \begin{equation}
        \Delta_{i}- \sqrt{\frac{2\sigma^{2}\log(1/\delta)}{u_{i}}} \geq c\Delta_{i}, \label{equ:delta}
    \end{equation}
     for some $c \in (0,1)$, where $c$ is chosen later.
      
     \noindent
     Since, $\mu_{1} = \min\limits_{j \in [p]}\mu_{ij} + \Delta_{i}$
    \begin{align}
    	&\hspace{-1mm}\Prob{\min\limits_{j \in [p]}\hat{\mu}_{ij}^{(u_{i})} + \sqrt{\frac{2\sigma^{2}\log(1/\delta)}{u_{i}}} \geq \mu_{1}} \nonumber \\
    	&= \Prob{\min\limits_{j \in [p]}\hat{\mu}_{ij}^{(u_{i})} - \min\limits_{j \in [p]}\mu_{ij} \geq \Delta_{i} - \sqrt{\frac{2\sigma^{2}\log(1/\delta)}{u_{i}}}} \nonumber\\
    	&\leq \Prob{\min\limits_{j \in [p]}\hat{\mu}_{ij}^{(u_{i})} - \min\limits_{j \in [p]}\mu_{ij} \geq c\Delta_{i}}. \nonumber\\
    	\intertext{Since nodes are independent for each channel,}
    	& = \prod\limits_{j \in [p]}\Prob{\hat{\mu}_{ij}^{(u_{i})} \geq \min\limits_{j \in [p]}\mu_{ij} + c\Delta_{i}} \nonumber \\
    	&\leq \Prob{\hat{\mu}_{ij_{u_{i}}^{*}}^{(u_{i})} \geq \mu_{ij^{*}} + c\Delta_{i}} \nonumber\\
    	&\leq \exp\left(-\frac{u_{i}c^{2}\Delta_{i}^{2}}{2\sigma^{2}}\right). \hspace{2mm} (\mbox{from Lemma \ref{lem:chernoffBound}}) \label{equ:secondcomponent} 
    \end{align}
     Adding Eq. (\ref{equ:firstcomponent}) and Eq. (\ref{equ:secondcomponent}),
     \begin{equation*}
         \Prob{G_{i}^{c}} \leq n\delta + \exp\left(-\frac{u_{i}c^{2}\Delta_{i}^{2}}{2\sigma^{2}}\right).
     \end{equation*}
     Substituting this in Eq. (\ref{equ:expT}), we have
     \begin{equation}
         \EE{T_{i}(n)} \leq u_{i} + n\left\{n\delta + \exp\left(-\frac{u_{i}c^{2}\Delta_{i}^{2}}{2\sigma^{2}}\right)\right\}. \label{equ:bounexpT}
     \end{equation}

    \item We choose $u_{i} \in [n]$ as $u_{i} = \left\lceil \frac{2\sigma^{2}\log(1/\delta)}{(1-c)^{2}\Delta_{i}^{2}}\right\rceil$ such that it satisfies Eq. (\ref{equ:delta}). This choice of $u_{i}$ can be larger than $n$ which makes the Eq. (\ref{equ:bounexpT}) trivially true as $T_{i}(n) \leq n$.
     \vspace{2mm}
     
     \item We find that the second term has polynomial dependence on $\delta$ (which depends on $n$) unless $\frac{c^{2}}{(1-c)^{2}} \geq 1$. However, if $c$ is chosen close to $1$, the first term will blow up. So we consider an arbitrary choice of $c \in (0,1)$ as $c = 1/2$.
     
     \vspace{2mm}
      Applying the choices of $u_{i}$ and $c$ in Eq. (\ref{equ:bounexpT}),
      \begin{align}
          \EE{T_{i}(n)} &\leq \left\lceil \frac{8\sigma^{2}\log(1/\delta)}{\Delta_{i}^{2}} \right\rceil + n(n\delta + \delta). \label{equ:boundexpTdelta}
        \intertext{Setting $\delta = 1/n^{2}$ in Eq. (\ref{equ:boundexpTdelta})}
        \EE{T_{i}(n)} &\leq 3 + \frac{16\sigma^{2}\log(n)}{\Delta_{i}^{2}},
        \hspace{0.2cm}\text{where $n \geq 1$}.
        \intertext{Therefore,} 
        \Regret_{n} &\leq 3\sum\limits_{i=1}^{m}\Delta_{i} + \sum\limits_{i: \Delta_{i} > 0} \frac{16\sigma^{2}\log(n)}{\Delta_{i}}. \qedhere
      \end{align}
    \end{itemize}
\end{proof}

\begin{proof}[\bf Proof of Theorem \ref{thm: Theorem 2}] 
    Let $\Delta > 0$ be some value to be tuned subsequently and from Theorem $1$ we can bound, 
    \begin{equation}
        \label{equ:finalexpT}
        \EE{T_{i}(n)} \leq 3 + \frac{16\sigma^{2}\log(n)}{\Delta_{i}^{2}}.
    \end{equation}
    Using regret decomposition, we have
    \begin{equation*}
        \Regret_{n} 
        = \sum\limits_{i: \Delta_{i} < \Delta}\Delta_{i}\EE{T_{i}(n)} + \sum\limits_{i: \Delta_{i} \geq \Delta}\Delta_{i}\EE{T_{i}(n)}.
    \end{equation*}
    Applying $\sum\limits_{i:\Delta_{i} < \Delta}T_{i}(n)\leq n$ and Eq. (\ref{equ:finalexpT})
    \begin{equation*}
        \implies \Regret_{n} \leq n\Delta + \sum\limits_{i: \Delta_{i} \geq \Delta}\left(3\Delta_{i} + \frac{16\sigma^{2}\log(n)}{\Delta_{i}}\right).
    \end{equation*}
    Choosing $\Delta = \sqrt{\frac{16m\sigma^{2}\log(n)}{n}}$, we get
    \begin{equation*}
        \Regret_{n} \leq 8\sqrt{nm\sigma^{2}\log(n)} + 3\sum\limits_{i=1}^{m}\Delta_{i}.  \qedhere
    \end{equation*}
\end{proof}

    \bibliographystyle{IEEEtran}
    \bibliography{ref}

\begin{thebibliography}{10}
\providecommand{\url}[1]{#1}
\csname url@samestyle\endcsname
\providecommand{\newblock}{\relax}
\providecommand{\bibinfo}[2]{#2}
\providecommand{\BIBentrySTDinterwordspacing}{\spaceskip=0pt\relax}
\providecommand{\BIBentryALTinterwordstretchfactor}{4}
\providecommand{\BIBentryALTinterwordspacing}{\spaceskip=\fontdimen2\font plus
\BIBentryALTinterwordstretchfactor\fontdimen3\font minus
  \fontdimen4\font\relax}
\providecommand{\BIBforeignlanguage}[2]{{%
\expandafter\ifx\csname l@#1\endcsname\relax
\typeout{** WARNING: IEEEtran.bst: No hyphenation pattern has been}%
\typeout{** loaded for the language `#1'. Using the pattern for}%
\typeout{** the default language instead.}%
\else
\language=\csname l@#1\endcsname
\fi
#2}}
\providecommand{\BIBdecl}{\relax}
\BIBdecl

\bibitem{AWPL2014_li2014frequency}
P.~K. Li, Z.~H. Shao, Q.~Wang, and Y.~J. Cheng, ``{Frequency-and
  pattern-reconfigurable antenna for multistandard wireless applications},''
  \emph{IEEE antennas and wireless propagation letters}, vol.~14, pp. 333--336,
  2014.

\bibitem{CM2015_mishra_smart}
D.~Mishra, S.~De, S.~Jana, S.~Basagni, K.~Chowdhury, and W.~Heinzelman,
  ``{Smart RF energy harvesting communications: Challenges and
  opportunities},'' \emph{IEEE Communications Magazine}, vol.~53, no.~4, pp.
  70--78, 2015.

\bibitem{Book19_Tor19Bandits}
T.~Lattimore and C.~Szepesv{\'a}ri, \emph{Bandit Algorithms}.\hskip 1em plus
  0.5em minus 0.4em\relax Cambridge University Press (to be printed soon),
  August 2020.

\bibitem{JSAC2015_energyharvesting_UlukusSennurYener}
S.~Ulukus, A.~Yener, E.~Erkip, O.~Simeone, M.~Zorzi, P.~Grover, and K.~Huang,
  ``{Energy harvesting wireless communications: A review of recent advances},''
  \emph{IEEE Journal on Selected Areas in Communications}, vol.~33, no.~3, pp.
  360--381, 2015.

\bibitem{CST2014_wirelessnetwork_LuXiaoPingDusit}
X.~Lu, P.~Wang, D.~Niyato, D.~I. Kim, and Z.~Han, ``{Wireless networks with RF
  energy harvesting: A contemporary survey},'' \emph{IEEE Communications
  Surveys \& Tutorials}, vol.~17, no.~2, pp. 757--789, 2014.

\bibitem{MM2014_harvesting_ValentaChristGreg}
C.~R. Valenta and G.~D. Durgin, ``{Harvesting wireless power: Survey of
  energy-harvester conversion efficiency in far-field, wireless power transfer
  systems},'' \emph{IEEE Microwave Magazine}, vol.~15, no.~4, pp. 108--120,
  2014.

\bibitem{TIT2017_online_sakulkar}
P.~Sakulkar and B.~Krishnamachari, ``{Online learning schemes for power
  allocation in energy harvesting communications},'' \emph{IEEE Transactions on
  Information Theory}, vol.~64, no.~6, pp. 4610--4628, 2017.

\bibitem{WN2018_darak2018distributed}
S.~J. Darak, C.~Moy, and J.~Palicot, ``{Distributed decision making policy for
  frequency band selection boosting RF energy harvesting rate in wireless
  sensor nodes},'' \emph{Wireless Networks}, vol.~24, no.~8, pp. 3189--3203,
  2018.

\bibitem{IA2017_maghsudi_distributed}
S.~Maghsudi and E.~Hossain, ``{Distributed user association in energy
  harvesting dense small cell networks: A mean-field multi-armed bandit
  approach},'' \emph{IEEE Access}, vol.~5, pp. 3513--3523, 2017.

\bibitem{AA-MAS2012_tran_long}
L.~Tran-Thanh, A.~Rogers, and N.~R. Jennings, ``{Long-term information
  collection with energy harvesting wireless sensors: a multi-armed bandit
  based approach},'' \emph{Autonomous Agents and Multi-Agent Systems}, vol.~25,
  no.~2, pp. 352--394, 2012.

\bibitem{INFOCOM2019_DistributedLearning_TibrewalPatchalaHanawal}
H.~Tibrewal, S.~Patchala, M.~Hanawal, and S.~Darak, ``{Distributed Learning and
  Optimal Assignment in Multiplayer Heterogeneous Networks},'' in \emph{IEEE
  INFOCOM}, 2019.

\bibitem{WiOpt2019_DistributedAlgorithms_VermaHanawalVaze}
A.~Verma, M.~Hanawal, and R.~Vaze, ``{Distributed Algorithms for Efficient
  Learning and Coordination in Ad Hoc Networks},'' in \emph{IEEE WiOpt}, 2019.

\bibitem{TWC2013_ju2013throughput}
H.~Ju and R.~Zhang, ``{Throughput maximization in wireless powered
  communication networks},'' \emph{IEEE Transactions on Wireless
  Communications}, vol.~13, no.~1, pp. 418--428, 2013.

\bibitem{GCC2014_ju2014user}
H.~Ju and R.~Zhang, ``{User cooperation in wireless powered communication
  networks},'' in \emph{2014 IEEE Global Communications Conference}.\hskip 1em
  plus 0.5em minus 0.4em\relax IEEE, 2014, pp. 1430--1435.

\bibitem{JSAC2015_yang2015throughput}
G.~Yang, C.~K. Ho, R.~Zhang, and Y.~L. Guan, ``{Throughput optimization for
  massive MIMO systems powered by wireless energy transfer},'' \emph{IEEE
  Journal on Selected Areas in Communications}, vol.~33, no.~8, pp. 1640--1650,
  2015.

\bibitem{ML2002_FiniteTimeAnlaysis_Auer}
P.~Auer, N.~Cesa-Bianchi, and P.~Fischer, ``{Finite-time Analysis of the
  Multiarmed Bandit Problem},'' \emph{Machine Learning}, vol.~47, no. 2--3, pp.
  235 --256, 2002.

\end{thebibliography}

\end{document}